\documentclass[pdflatex,sn-mathphys-ay]{sn-jnl}

\usepackage{graphicx}%
\usepackage{multirow}%
\usepackage{amsmath,amssymb,amsfonts}%
\usepackage{amsthm}%
\usepackage{mathrsfs}%
\usepackage[title]{appendix}%
\usepackage{xcolor}%
\usepackage{textcomp}%
\usepackage{manyfoot}%
\usepackage{booktabs}%
\usepackage{algorithm}%
\usepackage{algorithmicx}%
\usepackage{algpseudocode}%
\usepackage{listings}%
\usepackage{comment}
\usepackage{enumitem} 
\numberwithin{equation}{section}

\newcommand\norm[1]{\left\lVert#1\right\rVert}

\newcommand\prs[1]{\left(#1\right)}
\newcommand\sbk[1]{\left[#1\right]}

\newcommand{\E}{\mathbb{E}}

\newcommand{\Scal}{\mathcal{S}}

\newcommand{\R}{\mathbb{R}}

\newcommand{\gti}{\rightarrow\infty}
\newcommand{\gtz}{\rightarrow0}

\newcommand{\var}{\mathrm{Var}}

\theoremstyle{thmstyleone}%
\newtheorem{theorem}{Theorem}
\newtheorem{proposition}[theorem]{Proposition}%
\newtheorem{corollary}[theorem]{Corollary}

\theoremstyle{thmstyletwo}%
\newtheorem{remark}{Remark}%

\theoremstyle{thmstylethree}%
\newtheorem{assumption}{Assumption}
\newenvironment{acknowledgments}{%
  \section*{Acknowledgments}%
}{}
\begin{document}

\title[]{Deep Learning for Markov Chains: Lyapunov Functions, Poisson's Equation, and Stationary Distributions}


\author*[1]{\fnm{Yanlin} \sur{Qu}}\email{qu.yanlin@columbia.edu}

\author[2]{\fnm{Jose} \sur{Blanchet}}\email{jose.blanchet@stanford.edu}

\author[2]{\fnm{Peter} \sur{Glynn}}\email{glynn@stanford.edu}

\affil*[1]{\orgdiv{Columbia Business School}}

\affil[2]{\orgdiv{Stanford University}}


\abstract{
Lyapunov functions are fundamental to establishing the stability of Markovian models, yet their construction typically demands substantial creativity and analytical effort. In this paper, we show that deep learning can automate this process by training neural networks to satisfy integral equations derived from first-transition analysis. Beyond stability analysis, our approach can be adapted to solve Poisson’s equation and estimate stationary distributions. While neural networks are inherently function approximators on compact domains, it turns out that our approach remains effective when applied to Markov chains on non‑compact state spaces. We demonstrate the effectiveness of this methodology through several examples from queueing theory and beyond.}




\maketitle

\section{Introduction}\label{sec1}

The use of Lyapunov functions to settle stability questions in the setting of Markovian queues has become a standard tool within the research community. However, the construction of appropriate Lyapunov functions to determine the stability of many models remains challenging, and settling stability questions via the use of this method requires great creativity and insight in order to build suitable Lyapunov functions; see, for example, Chapter 16.4 of \cite{meyn2012markov} and Chapter 3.2 of \cite{dai2020processing}.

In this paper, we show how deep learning can be used to numerically compute Lyapunov functions. In particular, given a Markov chain $X=(X_n:n\geq0)$ living on state space $\Scal\subseteq \R^d$ with one-step transition kernel $P=(P(x,dy):x,y\in\Scal)$, the typical Lyapunov function used to settle stability involves a non-negative function $v=(v(x):x\in\Scal)$ satisfying the inequality
\begin{equation}
\label{lyapunov_eqn}
(Pv)(x)\leq v(x)-1,
\end{equation}
for $x\in A^c$ and some suitable set $A$, where $(Ph)(x)=\int_\Scal h(y)P(x,dy)$ for a generic function $h$; see Chapter 11.3 of \cite{meyn2012markov}. The function $v$ is then known to provide an upper bound on $\E_x T_A$, where $\E_x$ is the expectation on the path-space of $X$ conditioned on $X_0=x$, and $T_A=\inf\{n\geq0:X_n\in A\}$. In fact, our deep learning algorithm computes the minimal non-negative function $v$ satisfying \eqref{lyapunov_eqn}, namely the function $v^*(x)=\E_x T_A$ that satisfies \eqref{lyapunov_eqn} with equality.

More generally, we leverage deep learning to globally compute expectation functions of the form 
\begin{equation} 
\label{ustar_eqn}
u^*(x)=\E_x\sbk{\sum_{k=0}^T\exp\prs{-\sum_{j=0}^{k-1}\beta(X_j)}r(X_k)},
\end{equation}
for $r\geq0$ and $x\in C$, where $T=\inf\{n\geq0:X_n\in C^c\}$. Note that expected hitting times, exit probabilities, and infinite horizon expected discounted rewards (obtained by making $C^c$ the empty set) are all special cases of \eqref{ustar_eqn}. 
By conditioning on the first transition of the Markov chain (i.e. using {\it first-transition analysis} (FTA)), these functions can be characterized by integral equations.

We train neural networks to satisfy these equations by minimizing the integrated squared residual error with respect to the network parameters using stochastic gradient descent (SGD). To obtain an unbiased gradient estimator, we use the double-sampling trick that utilizes two independent and identically distributed (i.i.d.) copies of the first transition; see \cite{goda2023constructing}. We have used this idea in related work of ours on computing rates of convergence to equilibrium for Markov chains; see \cite{qu2024deep}. This idea has also been used to compute splitting probabilities and to predict rare events without simulating long sample paths, thereby achieving great efficiency; see \cite{strahan2023predicting,cheng2024surprising} and the references therein. In the context of reinforcement learning, the same idea can be used to perform policy evaluation, leading to the {\it residual gradient algorithm}; see \cite{baird1995residual,baird1998gradient,sutton2018reinforcement} (Chapter 11.5).

Our main contribution in this paper is to expose the use of deep learning to the queueing and applied probability communities, as a means of solving the linear integral equations that arise from Markov chain models. In the course of explaining these ideas, we describe how deep learning can be applied to compute global approximations to the solution of FTA equations arising within the general setting of \eqref{ustar_eqn}. As far as we are aware, this generalization is new. In particular, as explained above, our proposed approach therefore provides a vehicle for deep-learning-based numerical construction of Lyapunov functions. Hence,  deep learning can potentially play a role in settling the types of stability questions that often arise in Markov chain modeling applications.  

In addition to the integral equations that arise within the setting of FTA expectations and probabilities, this paper also contributes suitable deep-learning-based algorithms for solving the integral equations that arise within the context of stationary distribution computation and also that of Poisson's equation. To be specific, when $X$ is positive recurrent with a unique stationary distribution $\pi$, {\it Poisson's equation} involves computing the solution $u^*$ to the linear integral equation given by 
\begin{equation}
\label{poisson_eqn}
u(x)-(Pu)(x)=r(x)-\pi r,
\end{equation}
for $x\in\Scal$, where $\pi r=\int_\Scal r(y)\pi(dy)$ is the centering constant. Since this constant is typically unknown, the approach used to solve \eqref{ustar_eqn} must be adapted to address this issue; this contribution is also new. We further note that the unknown probability measure $\pi$ arising in the setting of stationary distribution computation is a measure rather than a function (as in \eqref{ustar_eqn}) and hence the deep learning algorithm for computing $\pi$ also requires a slight modification to the FTA deep learning approach we propose. Poisson's equation plays a key role in the central limit theorem for Markov chains, and in other calculations related to additive functionals of Markov chains; see, for example, \cite{maigret1978theoreme} and \cite{kurtz1981central}. 

The theory and practice of deep-learning-based function approximation requires that the domain of the functions being approximated be compact. We show how the deep learning method can be applied to Markov chains on non-compact state spaces, despite the fact that neural networks are  inherently function approximators on compact domains. An additional contribution of this paper is our careful analysis of the sample complexity for the deep learning approximations to the solutions to the equations arising in the computation of FTA expectations and probabilities. In particular, Theorem~\ref{thm:minimax:main} shows that, if the target function $u^*$ on $K\subseteq\R^d$ is $s$-H\"older continuous, then the mean squared error of the empirical risk minimizer (ERM) over clipped ReLU networks scales as $n^{-\frac{2s}{2s+d}}$ up to logarithmic factors, with an explicit conditioning factor reflecting the implicit fixed-point contraction associated with the linear integral equation arising in FTA settings. This connects the smoothness $s$ and dimension $d$ of the problem to the number of samples required to achieve a prescribed accuracy.

This paper is organized as follows: In Section \ref{first_section}, we apply deep learning to compute FTA-type quantities \eqref{ustar_eqn}, including Lyapunov functions. In Section \ref{subsec:minimax}, we develop the corresponding finite-sample guarantee. In Section \ref{poisson_section}, we apply deep learning to solve Poisson's equation \eqref{poisson_eqn}, without computing the centering constant. 
In Section \ref{stationary_section}, we apply deep learning to estimate stationary distributions.
In Section \ref{noncompact_section}, we extend our approach to non-compact state spaces. In Section \ref{numerical_section}, we demonstrate the effectiveness of our approach through several examples from queueing theory and beyond.

\section{First-transition analysis via deep learning}
\label{first_section}
We start by observing that the expectation $u^*(x)$ defined by \eqref{ustar_eqn} can be expressed as 
\begin{equation}
\label{inf_sum_eqn}
u^*(x)=\sum_{n=0}^\infty (H^ng)(x)
\end{equation}
for $x\in C$, where $H=(H(x,dy):x,y\in C)$ is the non-negative kernel defined via
\[
H(x,dy)=e^{-\beta(x)}P(x,dy)
\]
for $x,y\in C$ and 
\[
g(x)=r(x)+\int_{C^c}e^{-\beta(x)}P(x,dy)r(y)
\]
for $x\in C$. Here, $H^n=(H^n(x,dy):x,y\in C)$ is given by $H^0(x,dy)=\delta_x(dy)$ for $x,y\in C$ (where $\delta_x(\cdot)$ is a unit point mass at $x$), and
\[
H^n(x,dy)=\int_C H^{n-1}(x,dz)H(z,dy)
\]
for $n\geq1$.
The following assumption guarantees the finiteness of $u^*$; see Proposition \ref{finiteness_proposition} below.
\begin{assumption}
\label{finiteness_assumption}
There exists $v_1:C\rightarrow(0,\infty)$ such that $(Hv_1)(x)\leq v_1(x)-g(x)$ for $x\in C$.
\end{assumption}
With the Lyapunov-type Assumption \ref{finiteness_assumption} in force, the following bound on $u^*$ holds; see Theorem 14.2.2 in \cite{meyn2012markov}.
\begin{proposition}
\label{finiteness_proposition}
In the presence of Assumption \ref{finiteness_assumption}, $u^*(x)\leq v_1(x)$ for $x\in C$.
\end{proposition}
It follows from \eqref{inf_sum_eqn} and Proposition \ref{finiteness_proposition} that $u^*=(u^*(x):x\in C)$ is then a finite-valued solution of 
\begin{equation}
\label{FTA_eqn}
    u=g+Hu.
\end{equation}

Neural networks are flexible function approximators that can be trained to satisfy functional equations such as \eqref{FTA_eqn}. The key idea is to represent $u$ by a parametrized family of functions $\{u_\theta:\theta\in\Theta\}$. For example, a single-layer neural network takes the form
\begin{equation}
\label{nn_eqn}
u_\theta(x)=\sum_{k=1}^ma_j\sigma(w_j^\top x+b_j),\;\;x\in\R^d
\end{equation}
where $m$ is the number of hidden units, $\sigma:\R\rightarrow\R$ is a nonlinear activation function (such as the sigmoid function $\sigma(z)=1/(1+e^{-z})$), and the parameters $\theta=\{a_j,w_j,b_j\}_{j=1}^m$ include output weights $a_j$, input weights $w_j$, and biases $b_j$. 
If $u^*$ is known to be continuous and $C$ is compact, the universal approximation theorem \citep{cybenko1989approximation} asserts that the single-layer neural network above can approximate $u^*$ arbitrarily well, provided the activation function $\sigma$ is non-polynomial and the number of hidden units $m$ is sufficiently large.

Given a neural network $\{u_\theta:\theta\in\Theta\}$, our goal is therefore to choose $\theta$ such that $u_\theta$ approximately satisfies \eqref{FTA_eqn} on $C$, i.e., we need to minimize the extent to which $u_\theta$ fails to satisfy \eqref{FTA_eqn} on $C$. To be specific, we choose $\theta$ to minimize the integrated squared residual error
\[
l(\theta)=\int_C\prs{u_\theta(x)-g(x)-(Hu_\theta)(x)}^2\nu(dx),
\]
where $\nu$ is a probability measure on $C$ (e.g., the uniform distribution).
Let $X_0\sim\nu$ and $X_1\sim P(X_0,\cdot)$, where we write $Y\sim\eta$ to denote $P(Y\in\cdot)=\eta(\cdot)$. Then $l(\theta)$ can be written as an expectation
\[
\E\sbk{\prs{\E\sbk{u_\theta(X_0)-\Gamma(X_0,X_1)-e^{-\beta(X_0)}u_\theta(X_1)I(X_1\in C))\Big|X_0}}^2},
\]
where $\Gamma(X_0,X_1)=r(X_0)+e^{-\beta(X_0)}r(X_1)I(X_1\in C^c)$.
Stochastic gradient descent (SGD) $\theta_{t+1}=\theta_t-\alpha_t\hat{l}'(\theta_t)$ \citep{robbins1951stochastic} can be used to minimize $l(\theta)$, where $\alpha_t$ is the step size and $\hat{l}'(\theta_t)$ is an unbiased estimator of the gradient of $l(\theta)$ at $\theta_t$, i.e., $\E\hat{l}'(\theta_t)=l'(\theta_t)$. Since $l(\theta)$ is the expectation of a squared conditional expectation, differentiating it directly does not yield an unbiased gradient estimator (since the square function is nonlinear). This motivates rewriting $l(\theta)$ as a ``simple'' expectation. Given $X_0$, let $X_1$ and $X_{-1}$ be independent samples from $P(X_0,\cdot)$. Note that $l(\theta)$ becomes
\begin{equation}
\label{trick_eqn}
\begin{aligned}
&\E\sbk{u_\theta(X_0)-\Gamma(X_0,X_1)-e^{-\beta(X_0)}u_\theta(X_1)I(X_1\in C)}\\
&\cdot\sbk{u_\theta(X_0)-\Gamma(X_0,X_{-1})-e^{-\beta(X_0)}u_\theta(X_{-1})I(X_{-1}\in C)}.
\end{aligned}
\end{equation}
Differentiating the expression above yields a simple unbiased gradient estimator
\begin{equation}
\label{gradient_eqn}
\begin{aligned}
\hat{l}'(\theta)=&\sbk{u_\theta(X_0)-\Gamma(X_0,X_1)-e^{-\beta(X_0)}u_\theta(X_1)I(X_1\in C)}\\
&\cdot\sbk{u'_\theta(X_0)-e^{-\beta(X_0)}u'_\theta(X_{-1})I(X_{-1}\in C)}\\
&+\sbk{u_\theta(X_0)-\Gamma(X_0,X_{-1})-e^{-\beta(X_0)}u_\theta(X_{-1})I(X_{-1}\in C)}\\
&\cdot\sbk{u'_\theta(X_0)-e^{-\beta(X_0)}u'_\theta(X_1)I(X_1\in C)}
\end{aligned}
\end{equation}
where $u'_\theta$ is the gradient of the neural network with respect to its parameters, computed via backpropagation \citep{rumelhart1986learning}. Note that the two terms in $\hat{l}'(\theta)$ have the same expectation, so a computationally simpler unbiased gradient estimator is two times the first term. This simpler estimator avoids having to evaluate the gradient twice, a computationally expensive operation. We therefore adopt this simpler estimator in our algorithms. With $\hat{l}'(\theta)$ at hand, we can solve \eqref{FTA_eqn} by minimizing $l(\theta)$ via SGD $\theta_{t+1}=\theta_t-\alpha_t\hat{l}'(\theta_t)$. 
We refer to the resulting algorithm as FTA-RGA (Algorithm \ref{rga_alg}), where RGA stands for {\it residual gradient algorithm} (after \cite{baird1995residual}, who introduced this in the context of infinite horizon discounted reward policy evaluation). We refer to the algorithmic use of \eqref{trick_eqn} and \eqref{gradient_eqn} to obtain an unbiased gradient estimator $\hat{l}'(\theta)$ as the {\it double-sampling} trick.
\begin{algorithm}
\caption{FTA-RGA}\label{rga_alg}
\begin{algorithmic}
\Require probability measure $\nu$, neural network $\{u_\theta:\theta\in\Theta\}$, initialization $\theta_0$, step size $\alpha$, number of iterations $\bar{T}$
\For{$t\in\{0,...,\bar{T}-1\}$}
\State sample $X_0\sim\nu$ and $X_1,X_{-1}\stackrel{\text{iid}}{\sim}P(X_0,\cdot)$
\State Compute $\hat{l}'(\theta_t)$ as
\begin{align*}
    2&\sbk{u_{\theta_t}(X_0)-\Gamma(X_0,X_1)-e^{-\beta(X_0)}u_{\theta_t}(X_1)I(X_1\in C)}\\
    \cdot&\sbk{u'_{\theta_t}(X_0)-e^{-\beta(X_0)}u'_{\theta_t}(X_{-1})I(X_{-1}\in C)}
\end{align*}
\State $\theta_{t+1}=\theta_t-\alpha\hat{l'}(\theta_t)$
\EndFor
\State return $u_{\theta_{\bar{T}}}$
\end{algorithmic}
\end{algorithm}

We note that this idea can be generalized. Note that if $\tau(x)$ is a stopping time adapted to $(\mathcal{G}_n:n\geq0)$ (where $\mathcal{G}_n=\sigma(X_j:0\leq j\leq n)$), then for $x\in C$,
\begin{equation}
\label{generalized_ustar_eqn}
\begin{aligned}
    u^*(x)=&\E_x\sbk{\sum_{k=0}^{T\wedge\tau(x)-1}\exp\prs{-\sum_{j=0}^{k-1}\beta(X_j)}r(X_k)}\\
    &+\E_x\sbk{I(T>\tau(x))\cdot\exp\prs{-\sum_{j=0}^{\tau(x)-1}\beta(X_j)}u^*(X_{\tau(x)})}
\end{aligned}
\end{equation}
by the strong Markov property, where $a\wedge b:=\min(a,b)$. For $x\in C$, put 
\begin{align*}
    V(x)=&\sum_{k=0}^{T\wedge\tau(x)-1}\exp\prs{-\sum_{j=0}^{k-1}\beta(X_j)}r(X_k),\\
    W(x)=&I(T>\tau(x))\cdot\exp\prs{-\sum_{j=0}^{\tau(x)-1}\beta(X_j)}.
\end{align*}
It follows that an alternative to minimizing \eqref{trick_eqn} is to minimize $\tilde{l}(\theta)$, where
\begin{equation}
\label{tilde_eqn}
\begin{aligned}
\tilde{l}(\theta)=&\E\sbk{u_\theta(X_0)-V_1(X_0)-W_1(X_0)u_\theta(X_{1,\tau_1(X_0)})}\\
&\cdot\sbk{u_\theta(X_0)-V_{-1}(X_0)-W_{-1}(X_0)u_\theta(X_{-1,\tau_{-1}(X_0)})},
\end{aligned}
\end{equation}
where
\begin{align*}
    &(W_1(X_0),V_1(X_0),\tau_1(X_0),(X_{1,j}:0\leq j\leq\tau_1(X_0))),\\
    &(W_{-1}(X_0),V_{-1}(X_0),\tau_{-1}(X_0),(X_{-1,j}:0\leq j\leq\tau_{-1}(X_0)))
\end{align*}
are conditionally independent copies of $(W(X_0),V(X_0),\tau(X_0),(X_j:0\leq j\leq\tau(X_0)))$ given $X_0$, and $X_0$ is chosen randomly from the distribution $\nu$. Differentiating \eqref{tilde_eqn} yields an unbiased gradient estimator. There are several interesting algorithmic choices for $\tau(x)$, specifically $\tau(x)\equiv m$ and $\tau(x)$ given by the $m$'th time  $X$ hits a subset $K\subseteq C$. As we shall see later, this latter choice turns out to be useful in settings where $C$ is non-compact.

We observe that when $\tau(x)=\infty$, the algorithm involves running complete trajectories of $X$ up to time $T$. Even in that setting, our deep learning approach adds value, since it produces well-behaved global approximations to $u^*$ without either the need to explicitly discretize the solution over $C$ or to locally smooth the approximation.

\subsection{Solution smoothness}
The sample complexity theory we will develop in Section \ref{subsec:minimax} requires the knowledge that $u^*$ is sufficiently smooth over the domain of approximation. The smoothness of $u^*$ follows directly from Assumption \ref{finiteness_assumption} and the following Assumptions \ref{partial_H_assumption} and \ref{differentiable_assumption}. To prepare for stating Assumption \ref{partial_H_assumption}, let
\[
\norm{\varphi}_h=\sup\left\{\int_C\varphi(dy)q(y):|\varphi(y)|\leq h(y),\;y\in C\right\},
\]
for a positive function $h$ and a finite signed measure $\varphi$. We now assume that $C$ is convex (so that the limits appearing in the derivation below are well-defined).
\begin{assumption}
\label{partial_H_assumption}
For $1\leq i\leq d$ and $x\in C$, there exists a signed measure $\partial_i H(x,\cdot)$ such that
\[
\norm{\frac{H(x+\epsilon e_i,\cdot )-H(x,\cdot)}{\epsilon}-\partial_i H(x,\cdot)}_{v_1}\gtz
\]
as $\epsilon\gtz,$ where $e_i$ is the $i$-th basis vector, and $v_1$ is as in Assumption \ref{finiteness_assumption}.
\end{assumption}
\begin{assumption}
\label{differentiable_assumption}
$g(\cdot)$ is differentiable on $C$.
\end{assumption}
\begin{proposition}
\label{differentiable_proposition}
Under Assumptions \ref{finiteness_assumption},\ref{partial_H_assumption},\ref{differentiable_assumption}, $u^*(\cdot)$ is differentiable on $C$.
\end{proposition}
\begin{proof}
    Note that
    \[
        \frac{u^*(x+\epsilon e_i)-u^*(x)}{\epsilon}=\frac{g(x+\epsilon e_i)-g(x)}{\epsilon}+\int_C\prs{\frac{H(x+\epsilon e_i,dy)-H(x,dy)}{\epsilon}}u^*(y).
    \]
    But $|u_*|\leq v_1$ by Proposition \ref{finiteness_proposition}, so that Assumptions \ref{partial_H_assumption} and \ref{differentiable_assumption} ensure that
    \[
    \frac{u^*(x+\epsilon e_i)-u^*(x)}{\epsilon}\rightarrow \partial_i g(x)+\int_C \partial_i H(x,dy)u^*(y),
    \]
    proving the result.
\end{proof}
\section{Sample complexity}
\label{subsec:minimax}

\providecommand{\Rad}{\mathfrak{R}}
\providecommand{\Pdim}{\operatorname{Pdim}}

Suppose that we use SGD to minimize $\tilde{l}(\theta)$ as given by \eqref{tilde_eqn}, where $\tau(x)$ is chosen so that $X_{\tau(x)}\in K\subseteq C$ where $K$ is a compact subset of $\R^d$, and the probability $\nu$ is chosen to be supported on $K$. Note that the neural network approximation $u_\theta$ is then intended to approximate $u^*$ over $K$ (rather than $C$). (If $C$ is compact, $K=C$ is a possibility.) Then, $u^*$ is a fixed point of the equation
\begin{equation}
\label{eq:minimax:fta}
u(x)=\tilde{r}(x)+(\tilde{H}u)(x),
\end{equation}
for $x\in K$, where
\[
    \tilde{r}(x)=\E_x V(x),\;\;\tilde{H}(x,dy)=\E_x \sbk{W(x)I(X_{\tau(x)}\in dy)},\;\;\tilde{H}=(\tilde{H}(x,dy):x,y\in K).
\]
This section provides finite-sample guarantees for learning the solution to \eqref{eq:minimax:fta} on the subset $K$.
Our goal is to state the assumptions needed to understand the main quantitative result related to sample complexity, to explain their meaning, and to highlight how they can be met in applications.  We then present the theorem and discuss the roles of network architecture and smoothness. The proof is in Section \ref{subsec:minimax:proof}, which provides a concise proof roadmap in the spirit of the Bartlett–Bousquet–Mendelson (BBM) localization program~\cite{bartlett2005local,koltchinskii2006local}.

\begin{assumption}[Geometry, sampling, discount, and domination]\label{assump:minimax:dom}
\leavevmode
\begin{enumerate}[label=(\roman*)]
\item \(K\subset\mathbb{R}^d\) is compact with Lipschitz boundary, and the sampling measure \(\nu\) has a density on \(K\) bounded away from \(0\) and \(\infty\).
\item \(\tilde{r}\) is bounded on \(K\).
\item (Conditional moment of one–step contribution) There exists a finite constant \(\sigma_V^2<\infty\) such that for \(\nu\)-a.e. \(x\in K\),
\[
\mathbb{E}\big[V(x)^2\mid R=x\big] \;\le\; \sigma_V^2.
\]
\item (Killed–kernel domination) There exists \(\Lambda\in(0,\infty)\) such that $(\nu\tilde{H})(A)\leq\Lambda\nu(A)$ for all measurable $A\subseteq K$.
\item (Second–to–first moment control for the weight) There exists a finite constant \(C_W<\infty\) such that
\[
\mathbb{E}\big[W(x)^2\mid R=x,\,X_{\tau(x)}=y\big] \;\le\; C_W\,\mathbb{E}\big[W(x)\mid R=x,\,X_{\tau(x)}=y\big]
\text{ for \(\nu\text{-a.e. }x,y\in K\).}
\]
\end{enumerate}
\end{assumption}
Item~(i) is a mild regularity condition that aligns with standard approximation theory on \(K\); many variants are possible without changing the conclusions qualitatively. Item-(ii) is also mild and can be enforced by continuity. Item (iii) can be verified by a standard Lyapunov inequality.
Item~(iv) is a key quantitative ingredient: it says that the mass of \(\nu\) transported into any set \(A\) by $\tilde{H}$ is at most a \(\Lambda\)-fraction of \(\nu(A)\).  
Let $L^2(\nu)=\{u:K\rightarrow\R\;\text{such that}\;\int_K u^2(x)\nu(dx)<\infty\}$ and $\norm{u}_2^2:=\int_K u^2(x)\nu(dx)$ for $u\in L^2(\nu)$. Finally, if $M$ is a linear operator on $L^2(\nu)$, put $\norm{M}_{L^2(\nu)}:=\sup\{\norm{Mu}_2:\norm{u}_2=1\}$. Then, (iv) implies (see Proposition~\ref{prop:minimax:coercive}) that
\begin{equation}\label{eq:minimax:kappa}
\|\tilde{H}\|_{L^2(\nu)}\leq \sup_{x\in K}\tilde{H}(x,K)^{1/2}\cdot\sqrt{\Lambda},\;\;
\kappa:=1-\sup_{x\in K}\tilde{H}(x,K)^{1/2}\cdot\sqrt{\Lambda}\,\in (0,1),
\end{equation}
so that the fixed–point residual \(\|(I-\tilde{H})u\|_{L^2(\nu)}\) is equivalent to \(\|u-u^*\|_{L^2(\nu)}\) with constants depending primarily on~\(\kappa\) which capture how well conditioned is the empirical risk minimization problem. A practical way to ensure~(iv) is to impose an irreducibility/minorization structure and choose \(\nu\) using the Perron (right)–eigenfunction associated with the killed kernel; this route is classical in the theory of Harris chains and regeneration~\cite{nummelin1984general} and can be arranged in many models of interest; see also the constructions in~\cite{blanchet2016analysis}.  One may also assume an \(m\)-step analogue for some finite \(m\ge 1\): \(\nu(dx)\,H^m(x,dy)\le \Lambda\,\nu(dy)\) with \(\Lambda<1\), which yields a similar conclusion upon adjusting constants. Another way to establish~(iv) is via a Lyapunov condition, as described at the end of Section \ref{subsec:minimax:proof}. Item (v) is automatic with $C_W=1$ if $\beta(x) \geq 0$ .

\begin{assumption}[Smoothness]\label{assump:minimax:smooth}
The solution \(u^*\) belongs to the H\"older class \(C^s(K)\) of order \(s>0\) with radius \(L_s\).
\end{assumption}
We note that Proposition \ref{differentiable_proposition} provides conditions under which $u^*\in C^1(C)$ (and hence lies in $C^1(K)$).
\begin{assumption}[Hypothesis class and empirical risk]\label{assump:minimax:hyp}
We assume that a ReLU network architecture is used to approximate $u^*$. Precisely, given integers \(W,L\ge 1\) and \(B>0\), we let \(\mathcal{F}_{W,L,B}\) denote the class of ReLU networks (neural networks with activation function $\sigma(z)=\max(z,0)$) with at most \(W\) nonzero weights, depth at most \(L\), and outputs clipped to \([-B,B]\).  Given i.i.d. tuples \((R_i,V_i,W_i,Y_i,\tilde{V}_i,\tilde{W}_i,\tilde{Y}_i)_{i=1}^n\) such that, conditional on \(R_i\sim \nu\), $(V_i,W_i,Y_i)$ and $(\tilde{V}_i,\tilde{W}_i,\tilde{Y}_i)$ are two independent draws from \(P(R_i,\cdot)\) where $P(x,\cdot):=P_x((V(x),W(x),X_{\tau(x)})\in\cdot)$.
Put $\chi_i=(R_i,V_i,W_i,Y_i,\tilde{V}_i,\tilde{W}_i,\tilde{Y}_i)$, and define the one–step residual
\begin{equation*}
\ell_u(r,v,w,y):=u(r)-v-wu(y)
\end{equation*}
and the product loss
\begin{equation}\label{eq:minimax:empirical}
\begin{gathered}
h_u(r,v,w,y,\tilde{v},\tilde{w},\tilde{y}):=\ell_u(r,v,w,y)\,\ell_u(r,\tilde{v},\tilde{w},\tilde{y}),\\
\widehat{\mathcal{L}}_n(u):=\frac{1}{n}\sum_{i=1}^n h_u(\chi_i),\;\;\mathcal{L}(u):=\mathbb{E}[h_u(\chi_i)].
\end{gathered}
\end{equation}
Let \(\hat u_n\in\arg\min_{u\in\mathcal{F}_{W,L,B}}\widehat{\mathcal{L}}_n(u)\) be an empirical risk minimizer (ERM). The minimizer exists provided the parameters are restricted within a compact set, which we assume in our development. 
\end{assumption}
\textbf{Why the product loss?}  
Conditioning on \(R_i\), \((V_i,W_i,Y_i)\) and \((\tilde{V}_i,\tilde{W}_i,\tilde{Y}_i)\) are independent, so \(\mathcal{L}(u)=\mathbb{E}[\mathbb{E}[\ell_u(R_i,V_i,W_i,Y_i)|R_i)]^2]\).  Since \(u^*\) solves~\eqref{eq:minimax:fta}, \(\mathbb{E}[\ell_{u^*}(R_i,V_i,W_i,Y_i)|R_i]\equiv 0\) and hence
\begin{equation}\label{eq:minimax:risk-residual}
\mathcal{L}(u)=\|(I-H)(u-u^*)\|_2^2.
\end{equation}
Thus, minimizing \(\mathcal{L}\) amounts to minimizing the fixed–point residual in \(L^2(\nu)\); the product form in~\eqref{eq:minimax:empirical} yields an unbiased gradient estimator (Algorithm \ref{rga_alg}).

We now state the finite–sample guarantee.  It displays the usual bias–variance tradeoff: a smoothness–driven approximation term and a localized–complexity estimation term.  The fixed–point conditioning enters only through \(\kappa\).

\begin{theorem}[Finite–sample guarantee; minimax rate]\label{thm:minimax:main}
Suppose Assumptions~\ref{assump:minimax:dom}--\ref{assump:minimax:hyp} hold.  There exists a constant \(c_1<\infty\) depending only on \(B\), \(\|\tilde{r}\|_\infty\), $\|u^*\|_\infty$, $\sigma_W$, $C_W$, \(\Lambda\) such that, for all \(\delta\in(0,1)\), with probability at least \(1-\delta\),
\begin{equation}\label{eq:minimax:mainbound}
\| \hat u_n - u^*\|_{L^2(\nu)}^2
\le \frac{c_1}{\kappa^2}\,\Bigg(\underbrace{W^{-2s/d}}_{\text{approximation}} + \underbrace{\frac{W\,L\,\log W\cdot \log n + \log(1/\delta)}{n}}_{\text{estimation}}\Bigg).
\end{equation}
In particular, choosing \(L = [\log n]\) and \(W = [n^{\frac{d}{2s+d}}]\) gives (up to logarithmic factors)
\(\| \hat u_n - u^*\|_{L^2(\nu)}^2 = O(\kappa^{-2} n^{-2s/(2s+d)}\)), which is the minimax rate over H\"older classes.
\end{theorem}

\textbf{Beyond ReLU (architectures) and the role of smoothness.}
The approximation term in~\eqref{eq:minimax:mainbound} is controlled by how well the hypothesis class approximates \(C^s\).  For ReLU networks, the bound \(W^{-2s/d}\) with depth \(L\gtrsim \log W\) is standard; see, e.g.,~\cite{yarotsky2017error,yarotsky2018optimal,suzuki2018adaptivity,siegel2023optimal,Mhaskar1996}.  The analysis extends with minor changes to other piecewise–polynomial activations (same pseudo–dimension scaling) and to Barron–type two–layer networks, for which the integrated error can scale as \(O(1/n)\) when \(u^*\) belongs to a Barron class~\cite{Barron1993}.  Smoothness \(s\) dictates the approximation rate and hence the bias term; the variance term is driven by capacity via the pseudo–dimension bound recalled below.

\section{Solving Poisson's equation}
\label{poisson_section}
Let $X$ be a Markov chain with stationary distribution $\pi$, and suppose $r$ is $\pi$-integrable. Solving Poisson's equation 
\begin{equation}
\label{poi_eqn}
    u-Pu=r-\pi r
\end{equation} 
requires additional care, because the centering constant $\pi r=\int_\Scal r(y)\pi(dy)$ is often not available. Fortunately, since $\pi r$ is constant, it can be eliminated by applying $I-P$ to both sides, yielding
\begin{equation}
\label{tran_eqn}
u-2Pu+P^2u=r-Pr
\end{equation}
where $(P^2u)(x)=\E_xu(X_2)$. The following result shows that \eqref{tran_eqn} does not admit any additional solutions beyond those of \eqref{poi_eqn}.
Let $L^1(\pi)=\{u:\Scal\rightarrow\R\;\text{such that}\;\int_\Scal|u(x)|\pi(dx)<\infty\}$. 
\begin{proposition}
\label{poi_prop}
If \eqref{poi_eqn} has a unique solution in $L^1(\pi)$ that is unique within $L^1(\pi)$ up to an additive constant, the same is true for \eqref{tran_eqn}.
\end{proposition}
\begin{proof}
    Let $w$ be a solution of \eqref{tran_eqn}, i.e., $(I-P)^2w=(I-P)r$. Define $v=(I-P)w$. Then
    \[
    (I-P)v=(I-P)r\;\;\Rightarrow\;\;(I-P)(v-r)=0.
    \]
    Since $(I-P)u=r-\pi r$ has a unique solution up to a constant, $v-r$ must be a constant function.
    Then
    \[
    v(x)-r(x)=c\;\;x\in\Scal\Rightarrow\;\;((I-P)w)(x)=r(x)+c,\;\;x\in\Scal.
    \]
    By applying $\pi$ to both sides, $c$ must be $-\pi r$, so $w$ is a solution of \eqref{poi_eqn}.
\end{proof}
When $X$ is a positive recurrent Harris chain, the assumption of Proposition \ref{poi_prop} holds; see Proposition 1.1 of \cite{glynn1996liapounov}.
To minimize the integrated squared residual of \eqref{tran_eqn}, a simple unbiased gradient estimator can be constructed using the same double-sampling trick (Algorithm \ref{poi_alg}).
\begin{algorithm}
\caption{Poisson RGA}\label{poi_alg}
\begin{algorithmic}
\Require probability measure $\nu$, neural network $\{u_\theta:\theta\in\Theta\}$, initialization $\theta_0$, step size $\alpha$, number of iterations $\bar{T}$
\For{$t\in\{0,...,\bar{T}-1\}$}
\State sample $X_0\sim\nu$
\State sample i.i.d. $(X_1,X_2)$ and $(X_{-1},X_{-2})$
\State compute $\hat{l}'(\theta_t)$ as
\begin{align*}
    2&\sbk{u_{\theta_t}(X_0)-2u_{\theta_t}(X_1)+u_{\theta_t}(X_2)-r(X_0)+r(X_1)}\\
    \cdot&\sbk{u'_{\theta_t}(X_0)-2u'_{\theta_t}(X_{-1})+u'_{\theta_t}(X_{-2})}
\end{align*}
\State $\theta_{t+1}=\theta_t-\alpha\hat{l'}(\theta_t)$
\EndFor
\State return $u_{\theta_{\bar{T}}}$
\end{algorithmic}
\end{algorithm}
\section{Estimating stationary distributions}
\label{stationary_section}
Let $X$ be a positive recurrent Markov chain with stationary distribution $\pi$. Suppose the one-step transition kernel admits a bounded and continuous density $p$ with respect to a reference probability measure $\eta$, i.e., $P(x,dy)=p(x,y)\eta(dy)$ for $x,y\in\Scal$. The following result shows that $\pi$ must also admit a bounded and continuous density with respect to $\eta$, i.e., $\pi(dy)=\pi(y)\eta(dy)$. 
\begin{proposition}
\label{density_proposition}
If $p(x,y)$ is bounded and continuous in $y$, then $\pi$ has a density $\pi(\cdot)$ and $\pi(\cdot)$ is bounded and continuous.
\end{proposition}
\begin{proof}
    For any measurable $A\subseteq\Scal$, the global balance equation gives
    \[
    \pi(A)=\int_\Scal\pi(dx)P(x,A)=\int_\Scal\pi(dx)\int_A p(x,y)\eta(dy)=\int_A\prs{\int_\Scal\pi(dx)p(x,y)}\eta(dy).
    \]
    Hence $\pi\ll\eta$ with density
    \[
    \pi(y)=\int_\Scal\pi(dx)p(x,y)\leq\sup_{x,y\in\Scal}p(x,y)<\infty.
    \]
    If $y_n\rightarrow y$, then by the continuity of $p(x,\cdot)$ and the bounded convergence theorem,
    \[
    \pi(y_n)=\int_\Scal\pi(dx)p(x,y_n)\rightarrow\int_\Scal\pi(dx)p(x,y)=\pi(y),
    \]
    as $n\gti.$
\end{proof}
When $\pi$ is bounded and continuous and $\Scal$ is compact, we may use a neural network $\{\pi_\theta:\theta\in\Theta\}$ to approximate $\pi$, by minimizing the integrated squared residual of the global balance equation
\[
l(\theta)=\int_\Scal\prs{\pi_\theta(y)-\int_\Scal \pi_\theta(x)\eta(dx)p(x,y)}^2\nu(dy).
\]
where $\nu$ is a probability measure. 
Note that the minimizer of $l(\theta)$ is $\pi$ up to a multiplicative constant. To make the learning goal unique, we can enforce that $\pi_\theta$ takes a prescribed value (e.g., 1) at a specific location (e.g., 0), which can be done by training
\[
\tilde{\pi}_\theta(x)=\pi_\theta(x)-\pi_\theta(0)+1
\]
instead of $\pi_\theta(x)$.
Let $Y_0\sim\nu$ and $Z_1,Z_{-1}\stackrel{\text{iid}}{\sim}\eta$. Then
\[
l(\theta)=\E\sbk{\sbk{\pi_\theta(Y_0)-\pi_\theta(Z_1)p(Z_1,Y_0)}\sbk{\pi_\theta(Y_0)-\pi_\theta(Z_{-1})p(Z_{-1},Y_0)}},
\]
which can be minimized via SGD (Algorithm \ref{den_alg}) as before. Note that this approach is useful when the transition density is known. We note that $Z_1,Z_{-1}$ are sampled from the reference measure $\eta$, in contrast to the settings of the previous sections in which $X_1,X_{-1}$ are sampled from $P(X_0,\cdot)$.
\begin{algorithm}
\caption{Density RGA}\label{den_alg}
\begin{algorithmic}
\Require probability measure $\nu$, neural network $\{\pi_\theta:\theta\in\Theta\}$, initialization $\theta_0$, step size $\alpha$, number of iterations $\bar{T}$
\For{$t\in\{0,...,\bar{T}-1\}$}
\State sample $Y_0\sim\nu$ and $Z_1,Z_{-1}\stackrel{\text{iid}}{\sim}\eta$
\State $\hat{l}'(\theta_t)=2\sbk{\pi_{\theta_t}(Y_0)-\pi_{\theta_t}(Z_1)p(Z_1,Y_0)}\sbk{\pi'_{\theta_t}(Y_0)-\pi'_{\theta_t}(Z_{-1})p(Z_{-1},Y_0)}$
\State $\theta_{t+1}=\theta_t-\alpha\hat{l'}(\theta_t)$
\EndFor
\State return $\pi_{\theta_{\bar{T}}}$
\end{algorithmic}
\end{algorithm}
\section{Dealing with non-compact state spaces}
\label{noncompact_section}
As discussed earlier, an important application of our deep learning methodology is to the computation of Lyapunov functions. Such functions arise naturally to the analysis of stability questions for Markov chains taking values in a non-compact state space. Clearly, we can not expect an algorithm that terminates in finite time to globally approximate a function with an unbounded domain (because to encode the approximation then consumes infinite memory). Even the issue of whether the algorithm can faithfully approximate the solution over a compact set $K$ is not immediately obvious, because the solution on the compact set is influenced by the behavior of the Markov chain over the entire infinite state space. However, \eqref{generalized_ustar_eqn} and \eqref{tilde_eqn} make clear how this can be done. In particular, if we choose $\tau(x)=\inf\{n\geq1:X_n\in K\}$, \eqref{generalized_ustar_eqn} and \eqref{tilde_eqn} yield an algorithm that approximates $u^*$ over $K$ (Algorithm \ref{tau_alg}).
\begin{algorithm}
\caption{Non-compact FTA-RGA}\label{tau_alg}
\begin{algorithmic}
\Require probability measure $\nu$, neural network $\{u_\theta:\theta\in\Theta\}$, initialization $\theta_0$, step size $\alpha$, number of iterations $\bar{T}$
\For{$t\in\{0,...,\bar{T}-1\}$}
\State sample $X_0\sim\nu$
\State sample i.i.d. $(X_1,...,X_\tau)$ and $(X_{-1},...,X_{-\tau})$
\State compute $\hat{l}'(\theta_t)$ as
\begin{align*}
    2&\sbk{u_{\theta_t}(X_0)-V_1(X_0)-W_1(X_0)u_{\theta_t}(X_{1,\tau_1(X_0)})}\\
    \cdot&\sbk{u'_{\theta_t}(X_0)-W_{-1}(X_0)u'_{\theta_t}(X_{-1,\tau_{-1}(X_0)})}
\end{align*}
\State $\theta_{t+1}=\theta_t-\alpha\hat{l'}(\theta_t)$
\EndFor
\State return $u_{\theta_{\bar{T}}}$
\end{algorithmic}
\end{algorithm}
\section{Numerical illustrations}
\label{numerical_section}
\subsection{Stochastic fluid networks}
To illustrate the effectiveness of Algorithm \ref{rga_alg}, we apply it to a simple stochastic fluid network, the state space of which is compact.
We consider a network with two stations, each processing fluid workload at rate $r_1=r_2=1$. A fraction $p_{12}=0.4$ of the fluid processed at station 1 is routed to station 2. The arrival of work follows a compound renewal process: after each interarrival time $T\sim U[0,2]$, work $Z_1\sim U[0,1]$ and $Z_2\sim U[0,1.2]$ arrive at stations 1 and 2, respectively. This stochastic fluid network is stable as
\[
\E Z_1<r_1\E T,\;\;p_{12}\E Z_1+\E Z_2<r_2\E T.
\]
Let $X_n$ be the remaining workload vector immediately after the $n$-th arrival. To make the state space of this Markov chain compact, we impose a finite buffer capacity $c=5$ at each station (overflow is discarded), so $X_n\in[0,c]^2$.

We aim to compute $u^*(x)=\E_x T_A$ where $A=[0,1]^2$. This is a Lyapunov function that satisfies \eqref{lyapunov_eqn} with equality, i.e., $Pu^*=u^*-1$ on $A^c$. Establishing $Pu\leq u-1$ on $A^c$ analytically is nontrivial here because $A$ is not large enough to ignore the strong boundary reflection near the origin. Consequently, simple candidates such as $u(x)=x_1+x_2$ do not exhibit negative drift near $A$. In contrast, our approach directly computes the ``minimal'' Lyapunov function, which is the expected hitting time itself. 

To solve \eqref{FTA_eqn} with $\beta\equiv0$ and $r\equiv1$, we train a single-layer neural network \eqref{nn_eqn} with width $m=1000$ and sigmoid activation $\sigma(z)=1/(1+e^{-z})$. We run $\bar{T}=10^6$ SGD steps with step size $\alpha_t\equiv10^{-3}$. At each step, we sample $X_0$ uniformly from $A^c=[0,5]^2\setminus[0,1]^2$, simulate two i.i.d. first transitions $X_1,X_{-1}\stackrel{\text{iid}}{\sim}P(X_0,\cdot)$, compute the gradient estimator, and update the network parameters using Adam \citep{kingma2014adam}, a variant of SGD.
All computations are implemented in \texttt{PyTorch} \citep{paszke2019pytorch}.
To visualize the result, we evaluate the trained neural network $u_{\theta_{\bar{T}}}$ on a mesh grid in $A^c$ with spacing 0.1 (left plot of Figure \ref{figure_sfn}). To validate the result, we independently estimate $\E_x T_A$ for each $x$ on the same grid using 1000 i.i.d. simulation runs (right plot of Figure \ref{figure_sfn}).
\begin{figure}[ht]
    \centering
    \begin{minipage}{0.49\textwidth}
        \centering
        \includegraphics[width=0.9\linewidth]{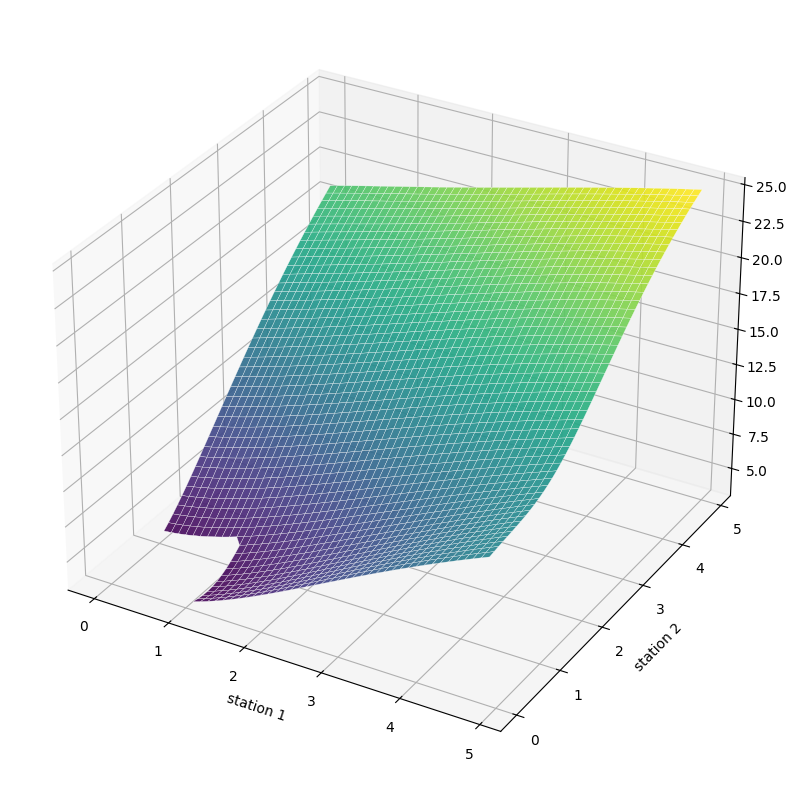}
    \end{minipage}
    \hfill
    \begin{minipage}{0.49\textwidth}
        \centering
        \includegraphics[width=0.9\linewidth]{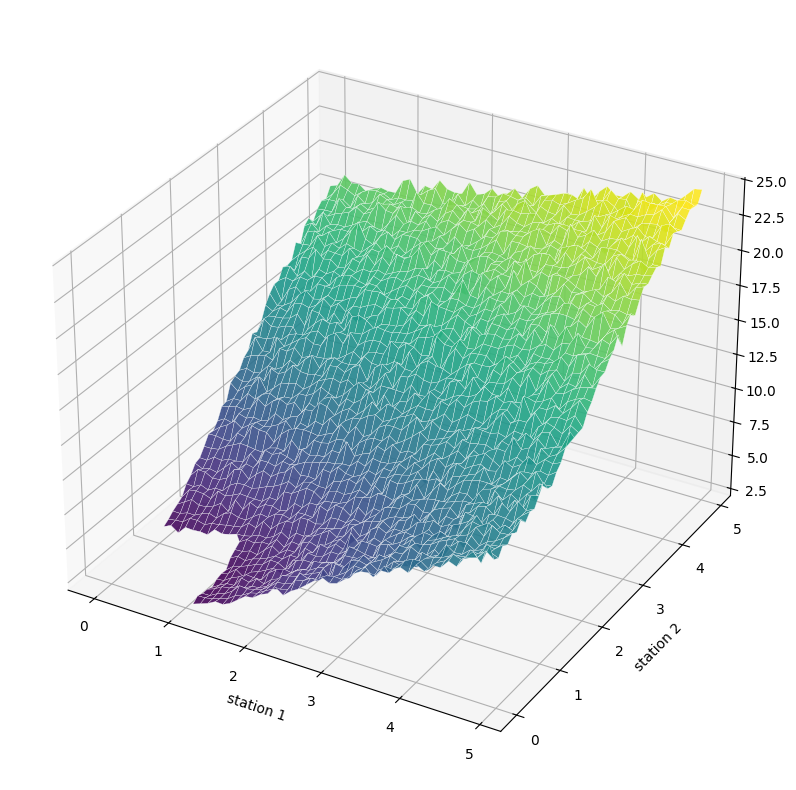}
    \end{minipage}
    \caption{Left: The learned solution of $Pu=u-1$. Right: The estimated $u^*(x)=\E_x T_A$ on a grid. 
    }
    \label{figure_sfn}
\end{figure}
\subsection{Kiefer-Wolfowitz workload vector}
As an example on a non-compact state space, we consider a G/G/2 queue. It is a single waiting line served by two parallel servers, with i.i.d. $U[0,2/0.6]$ interarrival times (arrival rate 0.6) and i.i.d. $U[0,2/0.5]$ service times (service rate 0.5). 
Customers are served in the order in which they arrive, by the first server to become available.
The Kiefer-Wolfowitz workload vector $W=(W^{\min},W^{\max})$ \citep{kiefer1956characteristics} records the remaining workload at each server immediately after each arrival, ordered from smallest to largest, providing a Markovian description of the system. To be specific,
\begin{align*}
W^{\min}_{n+1}=&\min((W^{\min}_n-A_{n+1})_++S_{n+1},(W^{\max}_n-A_{n+1})_+),\\
W^{\max}_{n+1}=&\max((W^{\min}_n-A_{n+1})_++S_{n+1},(W^{\max}_n-A_{n+1})_+),
\end{align*}
where $A_{n+1}$ is the interarrival time while $S_{n+1}$ is the service time. 

Since the state space $\Scal=\{(x_1,x_2):0\leq x_1\leq x_2\}$ is non-compact, we use Algorithm \ref{tau_alg} to compute $u^*(x)=\E_x T_A$ on $K$ where $A=\{(x_1,x_2):0\leq x_1\leq x_2\leq 3\}$ and $K=\{(x_1,x_2):3\leq x_1\leq x_2\leq 9\}$. 
Again, we train a single-layer neural network \eqref{nn_eqn} with width $m=1000$ and sigmoid activation $\sigma(z)=1/(1+e^{-z})$. We run $\bar{T}=10^6$ SGD steps with step size $\alpha_t\equiv10^{-3}$. At each step, instead of simulating two i.i.d. first transitions, we simulate two i.i.d. sample paths $X_1,...,X_\tau$ and $X_{-1},...,X_{-\tau}$ returning to $K$.
To visualize the result, we evaluate the trained neural network $u_{\theta_{\bar{T}}}$ on a mesh grid with spacing 0.2 in $K$ (left plot of Figure \ref{figure_kwv}). To validate the result, we independently estimate $\E_x T_A$ for each $x$ on the same grid using 10000 i.i.d. simulation runs (right plot of Figure \ref{figure_kwv}). From Figure \ref{figure_kwv}, we observe that the ``minimal'' Lyapunov function $u^*(x_1,x_2)$ remains nearly flat when $x_1$ is small and then smoothly transitions to linear growth as $x_1$ increases. This may explain why quadratic-then-linear (Huberized) Lyapunov functions work well in the stability analysis of queueing systems; see, e.g., \cite{blanchet2020rates}.
\begin{figure}[ht]
    \centering
    \begin{minipage}{0.49\textwidth}
        \centering
        \includegraphics[width=0.9\linewidth]{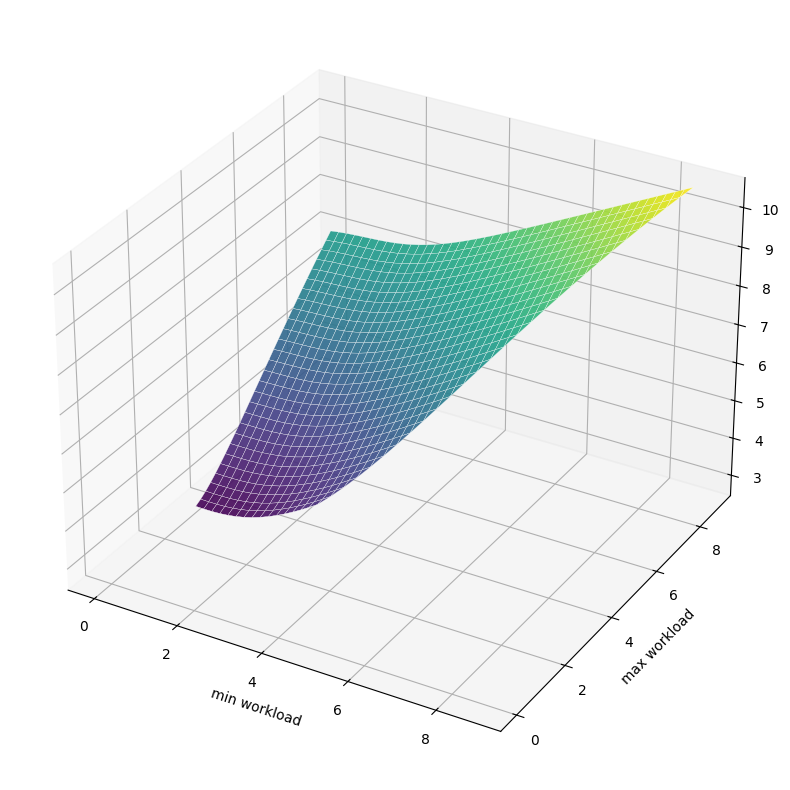}
    \end{minipage}
    \hfill
    \begin{minipage}{0.49\textwidth}
        \centering
        \includegraphics[width=0.9\linewidth]{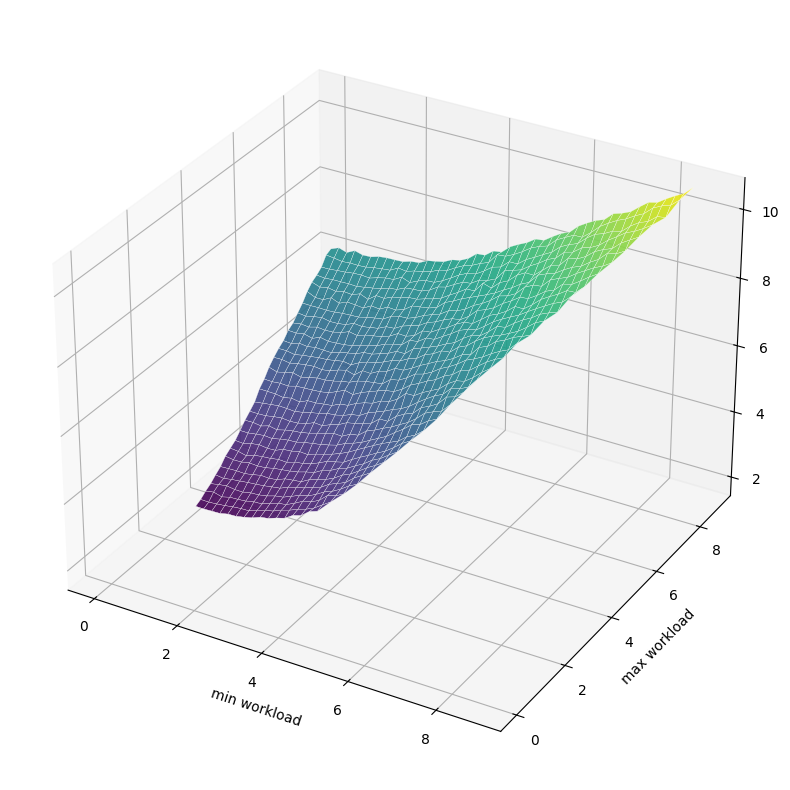}
    \end{minipage}
    \caption{Left: The learned solution of $Pu=u-1$. Right: The estimated $u^*(x)=\E_x T_A$ on a grid.
    }
    \label{figure_kwv}
\end{figure}

\subsection{Autoregressive processes}
To illustrate that deep learning can indeed be used to {\it correctly} solve Poisson's equation (Algorithm \ref{poi_alg}), we apply it to an autoregressive process, for which the solution (ground truth) can be analytically computed. Specifically, we consider the autoregressive Markov chain
\[
X_{n+1}=(X_n+Z_{n+1})/2,\;\;Z_{n+1}\stackrel{\text{iid}}{\sim}\mathrm{Ber}(1/2).
\]
This specific autoregressive process is known as the {\it Bernoulli convolution}; see \cite{erdos1939family}.
For $r(x)=x$, Poisson's equation \eqref{poi_eqn} is solved (up to an additive constant) by $u^*(x)=2x$.
For $r(x)=x^2$, Poisson's equation \eqref{poi_eqn} is solved (up to an additive constant) by $u^*(x)=(4/3)x^2+(2/3)x$. The verification of these solutions is straightforward and thus omitted.

We run Algorithm \ref{poi_alg} to solve Poisson's equation.
Since the state space is $[0,1]$, we train smaller neural networks with width $m=200$ using SGD with more aggressive step size $\alpha_t\equiv10^{-2}$.
The learned solutions, shown in Figure \ref{figure_poi}, match the corresponding true solutions.

\begin{figure}[ht]
    \centering
    \begin{minipage}{0.49\textwidth}
        \centering
        \includegraphics[width=0.9\linewidth]{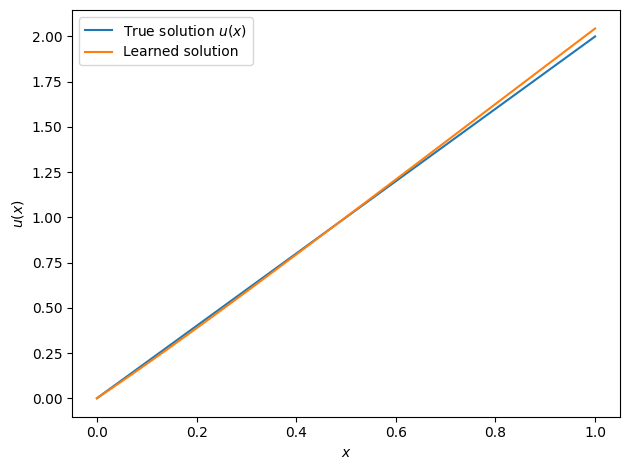}
    \end{minipage}
    \hfill
    \begin{minipage}{0.49\textwidth}
        \centering
        \includegraphics[width=0.9\linewidth]{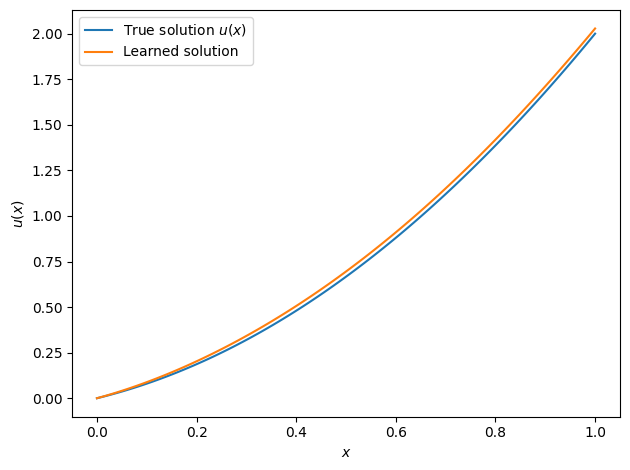}
    \end{minipage}
    \caption{The learned solutions of $u-Pu=r-\pi r$ with $r(x)=x$ (left) and $r(x)=x^2$ (right).
    }
    \label{figure_poi}
\end{figure}

\subsection{Markov chain Monte Carlo}
To illustrate that deep learning can indeed be used to {\it correctly} estimate stationary distributions (Algorithm \ref{den_alg}), we apply it to Markov chain Monte Carlo, for which the stationary distribution (ground truth) can be fully prescribed. Specifically, we consider the target density (unnormalized)
\[
\pi(x_1,x_2)\propto(2-x_1^2)(2-x_2^2)(2+x_1x_2),\;\;x_1,x_2\in[-1,1].
\]
We use a fixed-scan Gibbs sampler \citep{geman1984stochastic}: first update $X_1\sim\pi(\cdot|X_2)$ given the current $X_2$, then update $X_2\sim\pi(\cdot|X_1)$ given the new $X_1$, yielding a Markov chain that converges to $\pi$. The transition density is
\[
p(x,y)=\pi(y_1|x_2)\pi(y_2|y_1),\;\;x,y\in [-1,1]^2,
\]
where
\[
\pi(y_1|x_2)=(3/20)(2-y_1^2)(2+y_1x_2),\;\;\pi(y_2|y_1)=(3/20)(2-y_2^2)(2+y_1y_2).
\]
We run Algorithm \ref{den_alg} to estimate the stationary distribution of the Gibbs sampler. 
As mentioned in Section \ref{stationary_section}, implementing the Gibbs sampler is not needed to run Algorithm \ref{den_alg}, as it only requires the transition density.
In Figure \ref{figure_den}, we plot the learned stationary distribution on the left and the ground truth $\pi$ on the right.
\begin{figure}[ht]
    \centering
    \begin{minipage}{0.49\textwidth}
        \centering
        \includegraphics[width=0.9\linewidth]{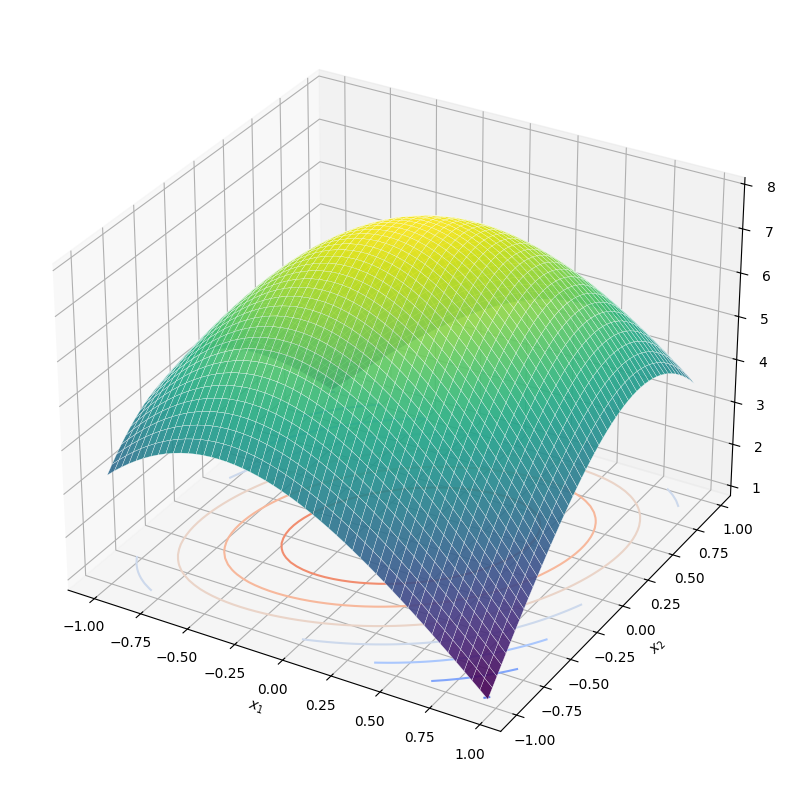}
    \end{minipage}
    \hfill
    \begin{minipage}{0.49\textwidth}
        \centering
        \includegraphics[width=0.9\linewidth]{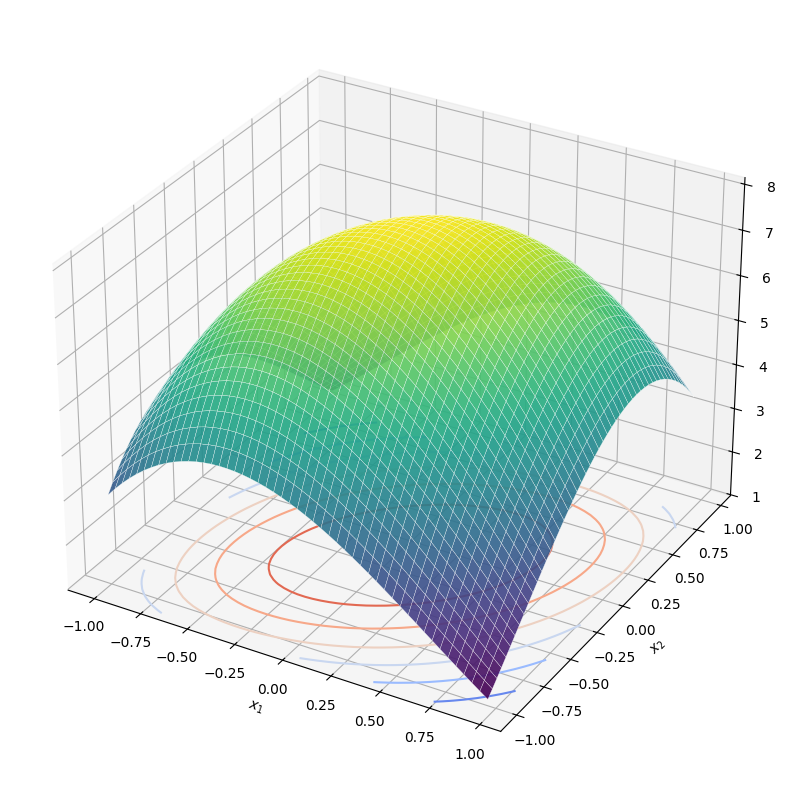}
    \end{minipage}
    \caption{Left: The learned stationary distribution of the Gibbs sampler. Right: The ground truth.
    }
    \label{figure_den}
\end{figure}
\section{Conclusions}
This paper demonstrates the potential of combining deep learning and simulation to study Markov chains. Under the unified FTA-RGA (first-transition analysis \& residual gradient algorithm) framework, we can construct Lyapunov functions, solve Poisson's equation, and estimate stationary distributions, for Markov chains on compact or non-compact state spaces. As we discussed earlier in our sample complexity analysis, we have left some considerations for future research, in particular, involving the role of the parameter $\Lambda$ which is connected with the choice of the sampling density and the number of steps required to implement the residual network. We expect that these issues will manifest themselves in the high dimensional setting prominently.

\section{Proof of Theorem \ref{thm:minimax:main}}
\label{subsec:minimax:proof}

\providecommand{\Rad}{\mathfrak{R}}
\providecommand{\Pdim}{\operatorname{Pdim}}

The proof follows the localization program of Bartlett–Bousquet–Mendelson (BBM)~\cite{bartlett2005local,koltchinskii2006local}.  We first record definitions and basic tools, then prove each proposition in turn, and finally assemble the oracle inequality.

\subsection{Preliminaries: definitions and tools}
\paragraph{Pseudo–dimension.}
For a class \(\mathcal{G}\) of real–valued functions on a set \(\mathcal{Z}\), the \emph{pseudo–dimension} \(\Pdim(\mathcal{G})\) is the largest integer \(m\) such that there exist points \(z_1,\ldots,z_m\in\mathcal{Z}\) and thresholds \(s_1,\ldots,s_m\in\mathbb{R}\) with the following property: for every labeling \(\sigma\in\{-1,+1\}^m\) there exists \(g\in\mathcal{G}\) with \(\operatorname{sign}(g(z_i)-s_i)=\sigma_i\) for all \(i\).  Pseudo–dimension reduces to VC–dimension for \(\{0,1\}\)–valued classes and controls combinatorial complexity for real–valued classes; see, e.g., \cite[Section 3]{bartlett2019nearly} and \cite{Langer2021}.

\paragraph{Empirical and Rademacher complexity.}
Given i.i.d.\ data \(Z_1,\ldots,Z_n\) with law \(P\) and a class \(\mathcal{G}\subset L^2(P)\), the \emph{empirical Rademacher complexity} is
\begin{equation*}
\Rad_n(\mathcal{G}) := \mathbb{E}\Bigg[\sup_{g\in\mathcal{G}} \frac{1}{n}\sum_{i=1}^n \varepsilon_i\, g(Z_i)\,\Big|\,Z_1,...,Z_n\Bigg],
\end{equation*}
where \(\varepsilon_1,\ldots,\varepsilon_n\) are i.i.d.\ Rademacher signs (i.e., uniform on $\{1,-1\}$) independent of the sample.  Symmetrization and contraction principles (see \cite[Sec.\ 2--3]{bartlett2005local}) convert population–empirical deviations into Rademacher averages and allow us to bound the latter via combinatorial parameters such as \(\Pdim\).

\paragraph{Local Rademacher complexity and the BBM fixed point.}
Let \(g_u:=h_u-h_{u^*}\) denote the \emph{excess loss}.  For \(r>0\), define the localized class
\begin{equation*}
\mathcal{G}(r):=\big\{ g_u : u\in\mathcal{F}_{W,L,B},\;\mathbb{E}[g_u]\le r\big\}.
\end{equation*}
Set \(\psi(r):=\mathbb{E}\,\Rad_n(\mathcal{G}(r))\).  The BBM theory shows that, under a Bernstein–type variance control (next paragraph), the estimation error is governed by the smallest \(\gamma^\star>0\) (the \emph{local fixed point}) solving \(\psi(\gamma^\star)\le c\,\gamma^\star\) for an absolute constant \(c>0\); see \cite[Theorem 3.3 \& 3.5]{bartlett2005local} and \cite[Section 3]{koltchinskii2006local}.  This yields a high–probability ``oracle inequality'' with \(\gamma^\star\) as the dominant complexity term.

\paragraph{Bernstein (variance–excess–risk) condition.}
A family \(\{g_u\}\) satisfies a (linear) Bernstein condition with constant \(C_{\rm B}\) if
\(
\operatorname{Var}(g_u)\le C_{\rm B}\,\mathbb{E}[g_u]
\)
for all \(u\) under consideration.  For square–loss–type problems, this is standard (the curvature of the loss forces variance to be controlled by excess risk).  In our product–loss setting, we verify such a condition directly below; see \cite{bartlett2005local,koltchinskii2006local}.

\subsection{Detailed proofs of the propositions}

\begin{proposition}[Coercivity]\label{prop:minimax:coercive}
Under Assumption~\ref{assump:minimax:dom}, \[\|\tilde{H}\|_{L^2(\nu)} \le \sup_{x\in K}\tilde{H}(x,K)^{1/2}\cdot\sqrt{\Lambda}.\]  Consequently, for all \(w\in L^2(\nu)\),
\begin{equation}\label{eq:minimax:equiv}
\kappa\|w\|_2 \le \|(I-\tilde{H})w\|_2 \le (1+\|\tilde{H}\|_{L^2(\nu)})\|w\|_2 \le 2\|w\|_2,
\end{equation}
with \(\kappa\) as in~\eqref{eq:minimax:kappa}.
\end{proposition}
\begin{proof}[Proof]
Note that $\tilde{H}$ is a substochastic kernel, so the Cauchy-Schwarz inequality implies that for any $f\in L^2(\nu)$,
\[
(\tilde{H}f)^2(x)\leq(\tilde{H}f^2)(x)\tilde{H}(x,K)
\]
for $x\in K$. So, part (iii) of Assumption \ref{assump:minimax:dom} implies that 
\[
\norm{\tilde{H}f}_2^2\leq\int_K\nu(dx)(\tilde{H}f^2)(x)\tilde{H}(x,K)\leq\Lambda\int_K\nu(dy)f^2(y)\cdot\sup_{x\in K}\tilde{H}(x,K)
\]
and hence \[\|\tilde{H}\|_{L^2(\nu)} \le \sup_{x\in K}\tilde{H}(x,K)^{1/2}\cdot\sqrt{\Lambda}.\]
The two–sided bound \eqref{eq:minimax:equiv} follows from \(\|(I-\tilde{H})w\|_2\ge (1-\|\tilde{H}\|_{L^2(\nu)})\|w\|_2=\kappa \|w\|_2\) and \(\|(I-\tilde{H})w\|_2\le \|w\|_2+\|H\|_{L^2(\nu)}\|w\|_2\le 2\|w\|_2\).
\end{proof}

\begin{proposition}[Bernstein / variance–excess–risk control]\label{prop:minimax:bernstein}
There exists a constant \(C_{\rm B}<\infty\) (depending only on \(B\), \(\sigma_V^2\), \(C_W\), \(\Lambda\), and \(\kappa\)) such that, for all \(u\in\mathcal{F}_{W,L,B}\),
\begin{equation*}
\operatorname{Var}\!\left(h_u(R,V,W,Y,\tilde V,\tilde W,\tilde Y)-h_{u^*}(R,V,W,\tilde V,\tilde W,\tilde Y)\right) \le C_{\rm B}\,\mathcal{L}(u).
\end{equation*}
Moreover, \(\mathcal{L}(u)=\|(I-\tilde H)(u-u^*)\|_2^2\) and \(|h_u|\) has finite second moment under Assumption~\ref{assump:minimax:dom}.
\end{proposition}

\begin{proof}[Proof (step–by–step)]
Write \(\varrho:=u-u^*\) and define
\[
m_{\varrho}(r):=\mathbb{E}[\ell_u(r,V,W,Y)-\ell_{u^*}(r,V,W,Y)\mid R=r],\qquad
s_{\varrho}(r):=\mathbb{E}[(\ell_u-\ell_{u^*})^2\mid R=r].
\]
Since \(u^*\) solves \eqref{eq:minimax:fta}, we have \(m_{\varrho}(r)=(I-\tilde H)\varrho(r)\), hence \(\mathcal{L}(u)=\mathbb{E}[m_{\varrho}(R)^2]\). Also,
\(
\ell_u-\ell_{u^*}=\varrho(R)-W\,\varrho(Y).
\)
Therefore,
\begin{align}
s_{\varrho}(R) &= m_{\varrho}(R)^2 + \operatorname{Var}(\varrho(R)-W\,\varrho(Y)\mid R) \\
&\le m_{\varrho}(R)^2 + \mathbb{E}[W^2\,\varrho(Y)^2\mid R].
\end{align}
By Assumption~\ref{assump:minimax:dom}(v), \(\mathbb{E}[W^2\,\varrho(Y)^2\mid R]\le C_W\,(\tilde H\,\varrho^2)(R)\). Integrating and using Assumption~\ref{assump:minimax:dom}(iv), we obtain
\[
\mathbb{E}[s_{\varrho}(R)] \le \mathbb{E}[m_{\varrho}(R)^2] + C_W\,\Lambda\,\|\varrho\|_2^2.
\]
By Proposition~\ref{prop:minimax:coercive}, \(\|\varrho\|_2\le \kappa^{-1}\,\|(I-\tilde H)\varrho\|_2=\kappa^{-1}\,\|m_{\varrho}\|_2\). Hence
\begin{equation}\label{eq:bernstein-second-moment-final}
\mathbb{E}[s_{\varrho}(R)] \le \Big(1 + C_W\,\Lambda\,\kappa^{-2}\Big)\,\mathcal{L}(u).
\end{equation}
Next, decompose the excess loss \(g_u:=h_u-h_{u^*}\) as
\(
g_u=(\ell_u-\ell_{u^*})\,\ell_u + \ell_{u^*}\,(\ell_u-\ell_{u^*}).
\)
Conditioning on \(R\) and using independence of \((V,W,Y)\) and \((\tilde V,\tilde W,\tilde Y)\) given \(R\),
\[
\mathbb{E}[g_u^2\mid R] \le 2\,s_{\varrho}(R)\,\Big(\mathbb{E}[\ell_u^2\mid R] + \mathbb{E}[\ell_{u^*}^2\mid R]\Big).
\]
Since outputs are clipped to \([-B,B]\), write \((a+b+c)^2\le 3(a^2+b^2+c^2)\) and note that, for \(\nu\)-a.e. \(R\),
\[
\ell_u^2 \le 3\,\big(u(R)^2 + V^2 + W^2 u(Y)^2\big),\qquad \ell_{u^*}^2 \le 3\,\big(u^*(R)^2 + V^2 + W^2 u^*(Y)^2\big).
\]
Let
\[
H_K := \sup_{x\in K} \tilde H(x,K) = \sup_{x\in K} \mathbb{E}[W\mid R=x],
\]
so the supremum is taken over \(x\in K\). By Assumption~\ref{assump:minimax:dom}(iii) and (v), and clipping \(|u|\le B\),
\[
\mathbb{E}[\ell_u^2\mid R] \le 3\,\big(B^2 + \sigma_V^2 + B^2\,\mathbb{E}[W^2\mid R]\big) \le 3\,\big(B^2 + \sigma_V^2 + B^2\,C_W\,H_K\big).
\]
Moreover, since \(u^*\) is continuous and $K$ is compact we use its sup-norm on \(K\):
\[
\|u^*\|_\infty := \sup_{x\in K} |u^*(x)| < \infty,
\]
so similarly
\[
\mathbb{E}[\ell_{u^*}^2\mid R] \le 3\,\big(\|u^*\|_\infty^2 + \sigma_V^2 + C_W\,H_K\,\|u^*\|_\infty^2\big).
\]
Thus we may take
\[
C' := 3\,\Big[(B^2+\|u^*\|_\infty^2) + \sigma_V^2 + C_W\,H_K\,(B^2+\|u^*\|_\infty^2)\Big],
\]
which yields the uniform bound
\[
\mathbb{E}[\ell_u^2\mid R] + \mathbb{E}[\ell_{u^*}^2\mid R] \le C'.
\]
Taking expectations and applying \eqref{eq:bernstein-second-moment-final} yields
\(
\mathbb{E}[g_u^2] \le C_{\rm B}\,\mathcal{L}(u),
\)
with \(C_{\rm B}:=2\,C'\,(1 + C_W\,\Lambda\,\kappa^{-2})\). Finally, \(\operatorname{Var}(g_u)\le \mathbb{E}[g_u^2]\) and \(\mathcal{L}(u^*)=0\).
\end{proof}

\begin{proposition}[Localized complexity and oracle inequality]\label{prop:minimax:oracle}
Let \(\mathcal{H}_{W,L,B}:=\{h_u:u\in\mathcal{F}_{W,L,B}\}\) and \(g_u=h_u-h_{u^*}\).  Denote by \(\gamma^\star\) the BBM fixed point for the localized class \(\mathcal{G}(r)=\{g_u:\mathbb{E}[g_u]\le r\}\).  Then there exist universal constants \(c_1,c_2>0\) such that, for all \(\delta\in(0,1)\), with probability at least \(1-\delta\),
\begin{equation}\label{eq:minimax:oracle}
\mathcal{L}(\hat u_n)-\mathcal{L}(u^*) \le c_1\Big(\inf_{u\in\mathcal{F}_{W,L,B}}\big[\mathcal{L}(u)-\mathcal{L}(u^*)\big]+\gamma^\star\Big)+c_2\,\frac{\log(1/\delta)}{n}.
\end{equation}
\end{proposition}

\begin{proof}[Proof (BBM road map with details)]
Set \(g_u=h_u-h_{u^*}\) and note that \(\mathbb{E}[g_u]=\mathcal{L}(u)\) by \eqref{eq:minimax:risk-residual}.  By the symmetrization inequality (see \cite[Sec.\ 2]{bartlett2005local}), for any \(r>0\),
\begin{equation}\label{eq:sym}
\mathbb{E}\Big[\sup_{g\in\mathcal{G}(r)}\big(\mathbb{E}[g]-\widehat{\mathbb{E}}_n[g]\big)\Big]
\le 2\,\mathbb{E}\,\Rad_n(\mathcal{G}(r)),
\end{equation}
where \(\widehat{\mathbb{E}}_n\) is the empirical mean.  Under the Bernstein condition (Proposition~\ref{prop:minimax:bernstein}), \(\operatorname{Var}(g)\le C_{\rm B}\,\mathbb{E}[g]\) for all \(g\in\mathcal{G}(r)\), so the localized class has \(L^2(P)\) radius \(O(\sqrt{r})\).  Applying the localized deviation bound of \cite[Theorem.\ 3.3 \& 3.5]{bartlett2005local} (see also \cite[Sec.\ 3]{koltchinskii2006local}) yields: there exist constants \(c_1,c_2>0\) such that, with probability at least \(1-\delta\), simultaneously for all \(u\),
\begin{equation}\label{eq:bbm-rel}
\mathbb{E}[g_u]\le \widehat{\mathbb{E}}_n[g_u] + c_1\,\psi(\mathbb{E}[g_u]) + c_2\,\frac{\log(1/\delta)}{n},
\end{equation}
where \(\psi(r)=\mathbb{E}\Rad_n(\mathcal{G}(r))\).  Let \(\gamma^\star\) solve \(\psi(\gamma^\star)\le \gamma^\star/(2c_1)\).  Then \eqref{eq:bbm-rel} implies, for all \(u\) with \(\mathbb{E}[g_u]\ge \gamma^\star\),
\(
\mathbb{E}[g_u] \le 2\,\widehat{\mathbb{E}}_n[g_u] + 2c_2\,\log(1/\delta)/n.
\)
Taking \(u=\hat u_n\) (an ERM of \(\widehat{\mathcal{L}}_n\)) and comparing to the best \(u\in\mathcal{F}_{W,L,B}\) (peeling on the event \(\mathbb{E}[g_{\hat u_n}]\ge \gamma^\star\) and its complement) yields~\eqref{eq:minimax:oracle}.  This is the standard BBM oracle inequality; see \cite[Theorem.\ 3.5]{bartlett2005local}.
\end{proof}

\begin{proposition}[Bounding the local fixed point]\label{prop:minimax:lrc}
Let \(d_{W,L}:=\Pdim(\mathcal{F}_{W,L,B})\).  There exists \(c>0\) such that
\begin{equation}\label{eq:minimax:lrc}
\gamma^\star \le c\,\frac{d_{W,L}\,\log n}{n}.
\end{equation}
\end{proposition}

\begin{proof}[Proof (from products to linear classes)]
Recall \(g_u=(\ell_u-\ell_{u^*})\ell_u + \ell_{u^*}(\ell_u-\ell_{u^*})\).  For any fixed sample \(\{(R_i,V_i,W_i,Y_i,\tilde V_i,\tilde W_i,\tilde Y_i)\}_{i=1}^n\), the multiplier inequality (a consequence of contraction; see \cite[Sec.~2--3]{bartlett2005local}) gives
\begin{equation}\label{eq:mult}
\Rad_n(\mathcal{G}(r)) \le \big(B_\ell+B_\ell^{\!*}\big)\,\Rad_n\!\left(\mathcal{L}_0\right),
\end{equation}
where 
\[
\mathcal{L}_0:=\big\{(r,v,w,y)\mapsto \ell_u(r,v,w,y)-\ell_{u^*}(r,v,w,y): u\in\mathcal{F}_{W,L,B},\ \mathbb{E}[g_u]\le r\big\}.
\]
Note that \(\ell_u-\ell_{u^*}\) is linear in \(\varrho=u-u^*\):
\(
\ell_u(r,v,w,y)-\ell_{u^*}(r,v,w,y)=\varrho(r)-w\,\varrho(y).
\)
Let
\[
H_K := \sup_{x\in K} \tilde H(x,K) = \sup_{x\in K} \mathbb{E}[W\mid R=x].
\]
Conditioning on the sample and using Cauchy--Schwarz for the weighted term,
\begin{align}
\Rad_n(\mathcal{L}_0)
&\le \frac{1}{n}\,\mathbb{E}_\varepsilon\sup_{\varrho}\sum_{i=1}^n \varepsilon_i\,\varrho(R_i)
\;+
\frac{1}{n}\,\mathbb{E}_\varepsilon\sup_{\varrho}\sum_{i=1}^n \varepsilon_i\,W_i\,\varrho(Y_i)\\
&\le \Rad_n\!\left(\mathcal{F}_{W,L,2B}^{\mathrm{eval}}\right)
\;+
\Big(\frac{1}{n}\sum_{i=1}^n W_i^2\Big)^{\!1/2}\,\Rad_n\!\left(\mathcal{F}_{W,L,2B}^{\mathrm{eval}}\right),
\label{eq:eval-weighted-split}
\end{align}
where \(\mathcal{F}_{W,L,2B}^{\mathrm{eval}}:=\{x\mapsto v(x): v\in\mathcal{F}_{W,L,2B}-\mathcal{F}_{W,L,2B}\}\) is an evaluation class (differences absorbed into the enlarged clipping range \(2B\)). Taking expectations over the sample and using Assumption~\ref{assump:minimax:dom}(v) together with \(\sup_x \mathbb{E}[W\mid R=x]=H_K\), we obtain
\begin{equation}\label{eq:weighted-factor}
\mathbb{E}\,\Rad_n(\mathcal{L}_0) \;\le\; \big(1+\sqrt{C_W\,H_K}\big)\,\mathbb{E}\,\Rad_n\!\left(\mathcal{F}_{W,L,2B}^{\mathrm{eval}}\right).
\end{equation}
Standard pseudo--dimension bounds (e.g., \cite[Sec.~3]{bartlett2019nearly}) yield
\begin{equation}\label{eq:pdim-rad}
\mathbb{E}\,\Rad_{m}(\mathcal{F}_{W,L,2B}^{\mathrm{eval}}) \;\leq\; 2Bc_0\,\sqrt{\frac{d_{W,L}\,\log(m/d_{W,L})}{m}},
\end{equation}
for some constant $c_0$. Combining \eqref{eq:mult}, \eqref{eq:weighted-factor}, and \eqref{eq:pdim-rad}, and applying the BBM localization calculus (variance \(\leq r\) up to constants, by Proposition~\ref{prop:minimax:bernstein}) gives
\[
\psi(r)\;\leq c(B_\ell+B_\ell^{\!*})\,\big(1+\sqrt{C_W\,H_K}\big)\,\sqrt{\frac{d_{W,L}\,\log n}{n}}\,\sqrt{r},
\]
for some constant $c$. Thus the fixed point \(\gamma^\star\) solving \(\psi(\gamma^\star)\le \gamma^\star/(2c_1)\) satisfies
\(
\gamma^\star \ = \; O(\frac{d_{W,L}\,\log n}{n}),
\)
as claimed.
\end{proof}

\begin{proposition}[Pseudo–dimension of ReLU networks]\label{prop:minimax:pdim}
For ReLU networks with at most \(W\) nonzero weights and depth \(L\), the pseudo–dimension satisfies
\begin{equation}\label{eq:minimax:pdim}
d_{W,L} = O(W\,L\,\log W),
\end{equation}
see \cite{bartlett2019nearly,Langer2021}.  Differences of networks and output clipping change the bound only by absolute constant factors.
\end{proposition}

\begin{proposition}[Approximation by ReLU and related networks]\label{prop:minimax:approx}
Under Assumption~\ref{assump:minimax:smooth}, for any \(\varepsilon\in(0,1)\) there exists \(u_\varepsilon\in\mathcal{F}_{W,L,B}\) with \(L\geq c_0\log(1/\varepsilon)\) for some $c_0>0$ and with a number of nonzero weights \(W = O(\varepsilon^{-d/s})\) such that
\begin{equation}\label{eq:minimax:approx}
\|u_\varepsilon-u^*\|_\infty\le \varepsilon.
\end{equation}
Consequently, \(\inf_{u\in\mathcal{F}_{W,L,B}}\|u-u^*\|_{L^2(\nu)}^2 = O(W^{-2s/d})\) when \(L \geq c_0\log W\).  See~\cite{yarotsky2017error,yarotsky2018optimal,suzuki2018adaptivity,siegel2023optimal,Mhaskar1996}.
\end{proposition}

\begin{proof}[Proof (sketch for completeness)]
The cited works construct deep ReLU networks that implement localized approximants with exponentially efficient reuse of parameters across dyadic partitions; this yields the \(W^{-s/d}\) uniform approximation rate with depth \(O(\log W)\).  Squaring and integrating over \(\nu\) yields the \(L^2(\nu)\) rate \(W^{-2s/d}\).
\end{proof}

\remark{The Barron–type alternative in \cite{Barron1993} provides \(O(1/n)\) rates under spectral smoothness (Barron norm) rather than H\"older smoothness. We mention this alternative as this type of regularity assumption is also common in the literature, but we do not pursue that here.}

\begin{proof}[Proof of Theorem~\ref{thm:minimax:main}]
By \eqref{eq:minimax:risk-residual} and Proposition~\ref{prop:minimax:approx}, there exists \(u_{W,L}\in\mathcal{F}_{W,L,B}\) with \(\mathcal{L}(u_{W,L}) = O(W^{-2s/d})\).  Apply the oracle inequality \eqref{eq:minimax:oracle} with \(u=u_{W,L}\) and use Proposition~\ref{prop:minimax:lrc} together with \eqref{eq:minimax:pdim} to obtain
\begin{equation*}
\mathcal{L}(\hat u_n) = O(W^{-2s/d})+O((W\,L\,\log W\cdot \log n+\log(1/\delta))/n).
\end{equation*}
Finally, use \eqref{eq:minimax:risk-residual} and the lower bound in \eqref{eq:minimax:equiv} with \(w=\hat u_n-u^*\) to translate \(\mathcal{L}(\hat u_n)\) into \(\|\hat u_n-u^*\|_{L^2(\nu)}^2\), yielding \eqref{eq:minimax:mainbound}.  Optimizing over \(W,L\) gives the displayed rate.
\end{proof}

\subsection{A Lyapunov condition to establish contraction}
Suppose that $\nu$ has a probability density $\eta=(\eta(y):y\in K)$ on $K$ (wrt Lebesgue measure) that is strictly positive and continuous (and hence bounded above and below). We further assume that $X$ has a transition density. In particular,  
\[
P(x,B)=\int_B p(x,y)dy
\]
for $x\in\Scal$ and (measurable) $B\subseteq \Scal$. If there exists a non-negative function $(q(x,y):x\in C-K,y\in K)$ such that
\[
\int_{C-K}e^{-\beta(x)}p(x,z)q(z,y)\leq q(x,y)-e^{-\beta(x)}p(x,y)
\]
for $x\in C-K$, $y\in K$, then the Lyapunov bound
\[
\E_x\sbk{\exp\prs{-\sum_{j=0}^{\tau_K-1}\beta(X_j)}I(X_{\tau_K}\in dy)}\leq q(x,y)dy
\]
for $x\in C-K$, $y\in K$ holds (where $\tau_K=\inf\{n\geq1:X_n\in K\}$). Hence, if there exists $c<1$ such that 
\[
e^{-\beta(x)}p(x,y)+\int_{C-K}e^{-\beta(x)}p(x,z)q(z,y)\leq c\eta(y)
\]
for $x,y\in K$, then
\[
\int_K\eta(x)dx\tilde{H}(x,dy)\leq c\eta(y)dy
\]
i.e., $\nu\tilde{H}\leq c\nu$, which yields the fact that $\tilde{H}$ is a contraction on $L^2(\nu)$, since
\[
\nu(\tilde{H}w)^2\leq \nu\tilde{H}w^2\leq c\nu w^2.
\]

\begin{acknowledgments}
J. Blanchet gratefully acknowledges support from DoD through the grants Air Force Office of Scientific Research under award number FA9550-20-1-0397 and ONR 1398311, also support from NSF via grants 2229012, 2312204, 2403007 is gratefully acknowledged.
\end{acknowledgments}




\bibliography{sn-bibliography}

\end{document}